\documentclass[preprint,12pt]{elsarticle}

\makeatletter
\def\ps@pprintTitle{%
\def\@oddfoot{\footnotesize\fbox{\parbox{\dimexpr\textwidth-2\fboxsep-2\fboxrule\relax}{%
   This article has been accepted for publication in \textit{Computer Networks}. This is the author's accepted manuscript version. Copyright may transfer without notice.}}}
 \let\@evenhead\@empty
 \let\@oddhead\@empty
 \let\@evenfoot\@oddfoot}
\makeatother

\usepackage{amssymb}
\usepackage{amsmath}
\DeclareMathOperator*{\argmax}{arg\,max}

\usepackage{amsthm}
\theoremstyle{plain}
\newtheorem{theorem}{Theorem}
\newtheorem{lemma}{Lemma}
\newtheorem{corollary}{Corollary}

\usepackage{booktabs}
\usepackage{multirow}
\usepackage{wrapfig}
\usepackage{xspace}
\usepackage{xcolor}
\usepackage[acronym]{glossaries}
\newacronym{ib}{IB}{Information Bottleneck}
\newacronym{dnn}{DNN}{deep neural network}
\newacronym{cnn}{CNN}{convolutional neural network}
\newacronym{kl}{KL}{Kullback-Leibler}
\newacronym{dpi}{DPI}{Data Processing Inequality}
\newacronym{fgsm}{FGSM}{Fast Gradient Sign Method}
\newcommand{\fgsm}{\gls{fgsm}\xspace}
\newacronym{bim}{BIM}{Basic Iterative Method}
\newcommand{\bim}{\gls{bim}\xspace}
\newacronym{mim}{MIM}{Momentum Iterative Method}
\newcommand{\mim}{\gls{mim}\xspace}
\newacronym{pgd}{PGD}{Projected Gradient Descent}
\newcommand{\pgd}{\gls{pgd}\xspace}
\newacronym{nes}{NES}{Natural Evolutionary Search}
\newcommand{\nes}{\gls{nes}\xspace}
\newcommand{\natk}{N-Attack\xspace}
\newcommand{\satk}{Square Attack\xspace}
\newacronym{eatk}{EVO}{Evolutionary Attack}
\newcommand{\eatk}{\gls{eatk}\xspace}
\newacronym{sopt}{S-OPT}{Sign-OPT}
\newcommand{\sopt}{\gls{sopt}\xspace}
\newacronym{hsja}{HSJA}{Hop-Skip-Jump Attack}

\newcommand{\tatk}{Triangle Attack\xspace}

\newacronym{ec}{EC}{Edge Computing}
\newacronym{db}{DB}{Distillated Bottleneck}
\newacronym{sb}{SB}{Supervised Bottleneck}
\newacronym{es}{ES}{Entropic Student}
\newacronym{bf}{BF}{BottleFit}
\newacronym{jc}{JC}{JPEG Compression}
\newacronym{qt}{QT}{Quantization}
\newacronym{sota}{SOTA}{state-of-the-art}
\newacronym{pac}{PAC}{Probably Approximately Correct}
\newacronym{asr}{ASR}{attack success rate}

\newacronym{lof}{LOF}{Local Outlier Factor}

\newcommand{\dnn}{\gls{dnn}\xspace}
\newcommand{\dnns}{\glspl{dnn}\xspace}
\usepackage[breaklinks=true,colorlinks]{hyperref}
\usepackage[capitalize,noabbrev]{cleveref}

\newcommand\blfootnote[1]{%
  \begingroup
  \renewcommand\thefootnote{}\footnote{#1}%
  \addtocounter{footnote}{-1}%
  \endgroup
}

\newcommand{\rev}[1]{\textcolor{black}{#1}}

\journal{Computer Networks}

\begin{document}


\begin{frontmatter}



\title{Adversarial Attacks to Latent Representations of Distributed Neural Networks in Split Computing}

\author[label1]{Milin Zhang}
\ead{zhang.mil@northeastern.edu}
\author[label1]{Mohammad Abdi}
\ead{abdi.mo@northeastern.edu}
\author[label2]{\\Jonathan Ashdown}
\ead{jonathan.ashdown@us.af.mil}
\author[label1]{Francesco Restuccia}
\ead{f.restuccia@northeastern.edu}

\affiliation[label1]{organization={Institute for the Wireless Internet of Things at Northeastern University},
            country={United States}}
            
\affiliation[label2]{organization={Air Force Research Laboratory},
            country={United States}}

\begin{abstract}
Distributed \glspl{dnn} have been shown to reduce the computational burden of mobile devices and decrease the end-to-end inference latency in edge computing scenarios\blfootnote{Approved for Public Release: Distribution Unlimited: AFRL-2025-2602.}. While distributed \glspl{dnn} have been studied, the resilience of distributed \glspl{dnn} to adversarial action remains an open problem. In this paper, we fill the existing research gap by rigorously analyzing the robustness of distributed \glspl{dnn} against adversarial action. We cast this problem in the context of information theory and rigorously proved that (i) the compressed latent dimension improves the robustness but also affect task-oriented performance; and (ii) the deeper splitting point enhances the robustness but also increases the computational burden. These two trade-offs provide a novel perspective to design robust distributed \dnn. To test our theoretical findings, we perform extensive experimental analysis by considering 6 different \gls{dnn} architectures, 6 different approaches for distributed \gls{dnn} and 10 different adversarial attacks using the ImageNet-1K dataset.
\end{abstract}




\date{}

\end{frontmatter}



\glsresetall

\section{Introduction}\label{intro}

\Glspl{dnn} have achieved significant success in various domains such as computer vision \cite{kirillov2023segment}, natural language processing \cite{openai2023gpt}, and wireless communication \cite{baldesi2022charm}, among many others. However, state-of-the-art \glspl{dnn} are challenging to deploy on resource-limited mobile devices. While mobile-specific \glspl{dnn} have been proposed \cite{sandler2018mobilenetv2}, they usually come with a significant loss in accuracy. On the other hand, completely offloading the computation to edge or cloud computers is impractical in mobile scenarios due to the excessive communication overhead corresponding to the transfer of the \gls{dnn} input from the mobile device to the edge/cloud \cite{wang2019edge}. A new paradigm called \textit{distributed computing} -- also referred to as \textit{split computing} in prior art \cite{matsubara2021split} -- divides the computation of \glspl{dnn} across multiple devices -- according to the available processing power and networking bandwidth. The key advantage is that optimal load distribution can be achieved while meeting maximum end-to-end latency constraints and also preserving the \gls{dnn} accuracy \cite{matsubara2022bottlefit}. 

Without loss of generality, we assume that a \gls{dnn} model is divided into a \emph{mobile DNN} and a \emph{local DNN}, respectively executed by the mobile device and an edge/cloud computer. Usually, the \gls{dnn} architecture is modified by introducing a compression layer (\textit{``bottleneck''}) at the end of the mobile DNN \cite{matsubara2019distilled,matsubara2020head}, which is trained to learn a latent representation that reduces the amount of data being sent to the edge/cloud. The compressed representation is then used by the local DNN to produce the final prediction output (e.g., classification). 

Although prior work has proven the advantages of distributing the \gls{dnn} computation, it is also evident that this approach opens the door to adversarial attacks to intermediate (latent) representations. The distributed nature of the computation exposes the latent representation to adversarial action. Indeed, due to the need for communicating the latent representation across devices over a wireless network, an adversary can easily eavesdrop the latent representation and craft an adversarial sample to compromise the local DNN as shown in \cref{fig:aml_sc}. 

\begin{figure}[t]
    \centering
    \includegraphics[width=\columnwidth]{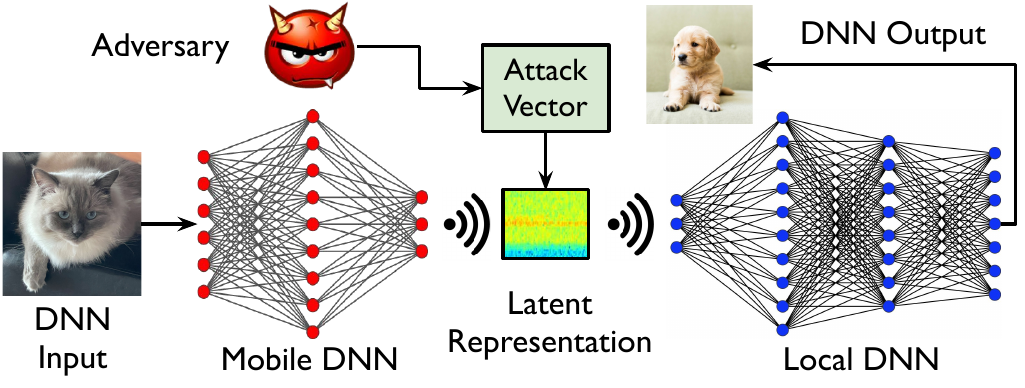}
    \caption{Overview of Adversarial Attacks to Distributed \glspl{dnn}.}
    \label{fig:aml_sc}
\end{figure}

Despite its significance and timeliness, assessing the robustness of distributed \glspl{dnn} remains an unexplored problem. We remark that achieving a fundamental understanding of these attacks and evaluating their effectiveness in state-of-the-art \glspl{dnn} is paramount to design robust distributed \dnns. To this end, we theoretically analyze the robustness of distributed \glspl{dnn} using \gls{ib} theory \cite{tishby2000information}. Our first key theoretical finding is that the latent robustness is highly depended on the depth of splitting point. With similar levels of information distortion, latent representations in deeper layers exhibit constantly better robustness compare to those in early layers. Our second key finding is that the trade-off between \dnn performance and robustness is intrinsically related to the dimension of the latent space. While a smaller latent space may increase robustness by reducing the model variance, it will also introduce bias in the model thus affecting the generalization capability of the \dnn model.

We extensively evaluate our theoretical findings by considering 10 adversarial algorithms, i.e., 4 white-box attacks \cite{goodfellow2014explaining,kurakin2018adversarial,dong2018boosting,madry2017towards} and 6 black-box attacks \cite{ilyas2018black,li2019nattack,andriushchenko2020square,dong2019efficient,cheng2019sign,wang2022triangle}. We apply these attacks to 6 different architectures \cite{simonyan2014very,he2016deep} designed with 6 distributed \gls{dnn} approaches \cite{eshratifar2019jointdnn, shao2020bottlenet++, matsubara2020head, matsubara2022bottlefit, singh2020end, matsubara2022supervised}. The experimental results validate our theoretical findings on the examined \glspl{dnn} and attack algorithms. 

\subsection*{Summary of Novel Contributions}

\noindent$\bullet$ We theoretically investigate the robustness of distributed \glspl{dnn} against adversarial action. We leverage notions of \gls{ib} theory and rigorously prove that the robustness of distributed \glspl{dnn} is affected by the splitting point and latent dimension. While existing work often optimize the splitting point and feature dimension to minimize end-to-end latency of distributed computing \cite{matsubara2022bottlefit,matsubara2022sc2}, for the first time we reveal that they are also key factors of robustness which need to be considered in designs;
 
\noindent$\bullet$ We perform extensive experiments with the ImageNet-1K \cite{deng2009imagenet} dataset, by considering 6 different \gls{dnn} architectures, 6 different distributed \gls{dnn} approaches under 10 different attacks to support our theoretical findings. The results show that the theoretical analysis applies to the experimental settings under consideration. We share our code for reproducibility at \url{https://github.com/Restuccia-Group/AdvLatent}, and we hope this work may open the door to a new field dedicated to studying the resilience of distributed \dnns. 

This paper is organized as follows. \cref{sec:model} define the threat model under consideration. Next, \cref{sec:theory} presents our theoretical analysis based on \gls{ib}. \cref{sec:exp-setup} discusses our experimental setup while \cref{sec:experiments} presents our experimental results. Finally, \cref{sec:rw} summarizes the related work and \cref{sec:conclusions} draws conclusions and discusses possible directions for future work.

\section{\rev{Overview of Distributed DNN}}\label{sec:overview}

\noindent\rev{\textbf{Distributed Neural Networks}. Deploying AI applications in mobile devices is challenging as there is no enough computation resources to execute large \dnns. Lightweight \dnns \cite{sandler2018mobilenetv2,tan2019mnasnet} have significant performance loss while offloading tasks to the edge device may incur in excessive latency \cite{yao2020deep}. Therefore, split computing, as an intermediate option, is proposed to accelerate large \dnns to resource-constrained devices \cite{kang2017neurosurgeon}. The key idea is to split a large \dnn into two parts -- a relative small head model deployed on the resource-limited device and a large tail model deployed on the edge device which has excessive computation power.}

\rev{``Partitioning optimization" and ``bottleneck optimization" represent two major research directions in split computing. Partitioning optimization focuses on selecting the optimal splitting depth within \dnns to achieve minimal inference latency under given computation and bandwidth constraints \cite{kang2017neurosurgeon,eshratifar2019jointdnn}. Splitting at early layers reduces computational overhead on the mobile device but increases communication burden, as latent representations in the initial layers of deep neural networks typically have larger dimensions. On the other hand, ``bottleneck optimization" aims to minimize the size of latent representations by introducing a ``bottleneck" layer before the split point to compress these representations \cite{matsubara2022bottlefit,eshratifar2019bottlenet}. This approach involves a critical trade-off: smaller bottlenecks may cause information loss during compression, degrading end-to-end performance, while larger latent dimensions incur unnecessary communication overhead. For a comprehensive review, readers are referred to \cite{zhang2025semantic}.}

\rev{While the depth and dimension of the partitioning layer are key design factors in split computing, existing research has overlooked how these parameters affect the adversarial robustness of split computing systems. In this work, we investigate the relationship between adversarial robustness in the latent space and both the dimension and depth of partitioning layers.} \smallskip

\noindent\rev{\textbf{Security and Privacy in Distributed DNNs}. Despite its effectiveness, this new cooperative paradigm of distributed \dnns introduces significant security and privacy concerns that require investigation. We roughly categorize the security and privacy issues in distributed \dnns based on the target device. On mobile devices, a critical issue is data privacy. Literature has demonstrated that adversaries can extract private input information of the mobile device from the intermediate outputs \cite{he2019model}. In addition, research has shown that poisoning input data can compromise data compression mechanisms at the mobile side, creating excessive transmission overhead \cite{zhang2024resilience}. Conversely, we focus on a new security issue that targets the tail model deployed on the edge device.}

\section{Threat Model}\label{sec:model}

\begin{figure}[t]
    \centering
    \includegraphics[width=\columnwidth]{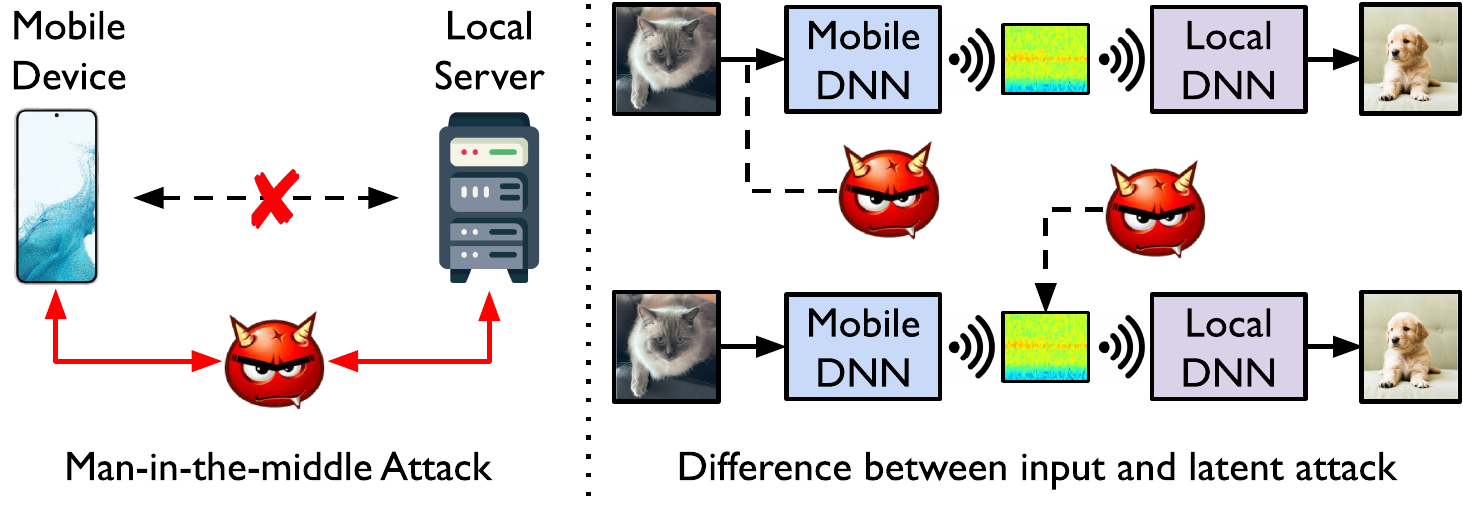}
    \caption{Threat model under consideration. (Left) The adversary plays a man-in-the-middle attack where the communication between mobile and local devices are altered without detection; (Right) Difference between adversarial attack in input space and adversarial attack in latent space.}
    \label{fig:threat_model}
\end{figure}

\noindent\textbf{Overview of the Adversarial Scenario.}~\cref{fig:threat_model} shows the threat model under consideration. We consider a man-in-the-middle attack to distributed \dnns where the mobile device and local server are connected through an insecure network. In this scenario, attackers can intercept transmissions and modify communications between mobile and edge endpoints without detection. More specifically, the attacker can add unperceptive perturbations to latent representations to mislead the local \dnn at the edge, compromising the task-oriented performance of the distributed computing system. This attack differs fundamentally from traditional adversarial attacks in input space as latent representations often contains richer semantic information (e.g. objects and shapes) and less task-irrelevant information (e.g., background noise) \cite{shwartz2017opening}, resulting in distinctive robustness characteristics compared to raw input data.

Analog to conventional adversarial attacks in the input space, we consider $l_p$ bounded attacks in our threat model. Remark that while $l_p$ constraints are first proposed to model the human invisibility in vision tasks~\cite{simonyan2014very}, it has been widely applied in diverse tasks that operate without human involvement, such as cybersecurity~\cite{rosenberg2021adversarial} and wireless communication~\cite{adesina2022adversarial}. This is because the $l_p$ distortion represents a worst-case scenario for resilience. Thus, we believe $l_p$ bounded attacks is an essential starting point to assess the robustness of distributed \dnns. 

\noindent\textbf{Adversarial Attacks in Input Space.}~Let $f: \mathbb{R}^{d} \mapsto \mathbb{C}^{k}$ denote a \dnn where $\mathbb{R}$ and $\mathbb{C}$ are respectively the input and output space, and $d$ and $k$ are the corresponding dimension of these two spaces. The \dnn will assign highest score to the correct class $y = \argmax_{k}f(x)$ for each input $x$. The adversarial goal is to introduce a perturbation $\delta_d \in \mathbb{R}^d$ to the original sample so that
\begin{equation}
    \argmax_{k=1,...,K}f(x+\delta_d) \neq y,
\end{equation}
where $||\delta_d||_{p} \leq \sigma$ and $\sigma$ is the distance constraint under different $l_{p}$ norm. Additionally, for visual applications, $\delta_d$ should satisfy the condition $x + \delta_d \in [0, 1]^d$ as there is an explicit upper and lower bound for red, green and blue (RGB) value in digital images.

\noindent\textbf{Adversarial Attacks in Latent Space.}~Let $g:\mathbb{R}^d \mapsto \mathbb{H}^t$ and $f:\mathbb{H}^t \mapsto \mathbb{C}^k$ denote the mobile \dnn and local \dnn, where $\mathbb{H}$ and $t$ are the latent space and its associated dimension, respectively. For each input $x$, the mobile \dnn will generate a corresponding latent representation $g(x) \in \mathbb{H}^t$ and the local \dnn will generate output $y = arg\max_{k}f(g(x))$ by taking the latent representation as input. Adversarial action in latent space adds a perturbation $\delta_t \in \mathbb{H}^t$ such that
\begin{equation}
    \argmax_{k=1,...,K}f(g(x)+\delta_t) \neq y,
\end{equation}
where $||\delta_t||_p \leq \sigma$ is the distance constraint under $l_p$ norm. We remark that the latent representations are model-dependent and there is no explicit bound for their value other than their computer-level representation (e.g., float, integer, double). 

\section{Theoretical Analysis}\label{sec:theory}

\subsection{Analysis through Information Bottleneck}
\begin{figure}[t]
    \centering
    \includegraphics[width=\columnwidth]{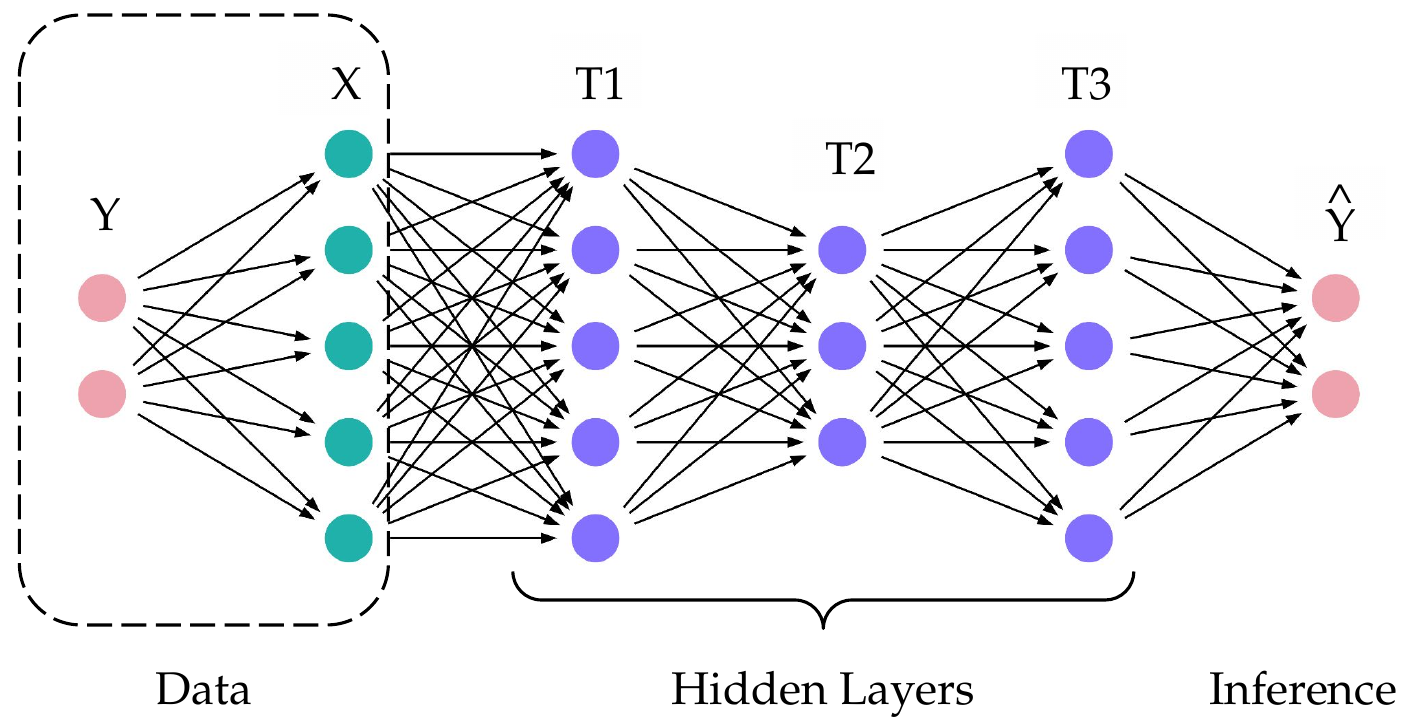}
    \caption{Modeling DNN with IB. Each representation $T_i$ only depends on the previous output $T_{i-1}$, and the optimal $T^{*}_{i}$ can be interpreted as the IB which optimizes \cref{eqn:ib} at layer $i$.}
    \label{fig:ibdnn}
\end{figure}
The \acrfull{ib} is a model-agnostic information-theoretical framework introduced by \cite{tishby2000information} to extract the relevant information about a random variable (r.v.) $Y$ from another r.v.~$X$ by finding a representation $T$ which compresses the information of $X$ while captures only the sufficient information about $Y$. As shown in Figure~\ref{fig:ibdnn}, we model a \gls{dnn} with a Markov chain $Y\mapsto X\mapsto T_{1}\mapsto \cdots \mapsto T_{k} \mapsto \hat{Y}$, where $X$, $Y$, $\hat{Y}$ and $T_i$ are respectively the input, its label, the inference output and the output of the $i$-th hidden layer. The \gls{ib} optimizes the following:
\begin{equation}
    \min_{P(T_i|X)} I(X;T_i)-\beta \cdot I(Y;T_i), 1 \le i \le k  \label{eqn:ib}
\end{equation}
where $I(X;T_i)$ is the mutual information between $X$ and $T_i$ while $I(Y;T_i)$ is the mutual information between $Y$ and $T_i$. Each layer can be thus described by its own unique information plane $(I(X;T_i), I(Y;T_i))$ which represents its compression and generalization capability. Notice that optimizing Equation~\ref{eqn:ib} is equivalent to minimizing $I(X;T_i)$  -- i.e., learning to compress -- whereas maximizing $I(Y;T_i)$ -- i.e., learning to generalize. To simplify notation, and without loss of generality, henceforth we will consider a single generic hidden layer $T$.

\begin{lemma}[Latent Robustness]\label{def:lr}
    For a given \gls{dnn}, the mutual information $I(Y;T)$ quantifies the robustness at the layer $T$.
\end{lemma}
\begin{proof}
The end-to-end robustness of the \gls{dnn} model can be modeled with $I(Y;\hat{Y})$, which measures the mutual information between the label $Y$ and the \dnn inference $\hat{Y}$ \cite{tishby2015deep}. We apply the \gls{dpi} to describe the information loss during processing \cite{cover1999elements}:
\begin{equation}
    I(Y;X) \geq I(Y;T) \geq I(Y;\hat{Y}) \label{eqn:dpi}
\end{equation}

In short, the generalization metric $I(Y;T)$ of hidden layer $T$ also describes the upper bound of $I(Y;\hat{Y})$, which is intrinsically a measure of robustness at layer $T$.     
\end{proof}

By assuming that adversarial perturbations are not observable (i.e., adversarial perturbation doesn't change the ground-truth of mutual information), it follows there is a prior yet unknown optimal solution $I^*(Y;T)$ for a specific \gls{dnn} architecture that satisfies the \gls{ib} where adversarial perturbations cannot decrease the performance -- in other words, $I^*(Y;T)$ is resilient to adversarial attacks. The key issue is that although each \gls{dnn} has a hypothesis set defined by its parameters, the optimum parameter set exhibiting the largest $I^*(Y;T)$ is unknown. To this end, each trained \dnn using a finite dataset $(X,Y)$ has its own estimation $I(Y;T)$ and adversarial perturbations can decrease the estimated mutual information by introducing distribution shift in dataset.
\begin{theorem}[Compression-Robustness Trade-Off] \label{thm:1}
    The adversarial robustness $I(Y;T_{adv})$ is jointly determined by its ideal generalization capability $I^*(Y;T)$ and dimensionality $\mathcal{O}( |\mathcal{T}||\mathcal{Y}| / \sqrt{n})$:
    \begin{equation}
        I(Y;T_{adv}) = I^*(Y;T) - \mathcal{O}(\frac{|\mathcal{T}||\mathcal{Y}|}{\sqrt{n}}), \label{eqn:advb}
    \end{equation}
    where $n$ denotes the number of data and $|\mathcal{T}|$, $|\mathcal{Y}|$ are the cardinality of $T$ and $Y$, respectively.
\end{theorem}

\begin{proof}
    \cite{shamir2010learning} have proven that the estimated mutual information $I(Y;T)$ using a dataset $(X,Y)$ with finite samples has the following error bound:
    \begin{equation}
        ||I^*(Y;T) - I(Y;T)||\leq \mathcal{O}(\frac{|\mathcal{T}||\mathcal{Y}|}{\sqrt{n}}). \label{eqn:fsb}
    \end{equation}
    
    As the adversarial examples aim to decrease the performance, the adversarial robustness $I(Y;T_{adv})$ approaches the lower bound in \cref{eqn:fsb}. Thus, \cref{eqn:advb} holds.
\end{proof}

\cref{thm:1} is the information version of the complexity-generalization trade-off in \gls{pac} learning. A larger latent space $|\mathcal{T}|$ (i.e., a more complex hypothesis set in \gls{pac} learning) will have larger variance resulting in decreased performance with inputs coming from a distribution different than $X$. Conversely, with a smaller latent space (i.e., a smaller hypothesis set), the \dnn has more bias, which leads to less accuracy (i.e., a smaller $I^*(Y;T)$). 

Since the \gls{ib} is model-agnostic and generalized, we point out that our theoretical framework alo provides value to explain the robustness in the general adversarial setting. For example, \cite{simon2019first} derived a similar theoretical finding that robustness of \dnns decreased with growing data dimension. However, they solely concentrate on dimensionality while our analysis delves into the trade-off. \cite{su2018robustness} and \cite{tsipras2018robustness} empirically explored the trade-off between robustness and performance, a concept close to \cref{thm:1}. However, we offer a rigorous theoretical explanation of this trade-off, revealing that dimensionality serves as the determining factor.

In distributed \dnns, bottlenecks are often introduced to compress the latent dimension in order to reduce the communication overhead. \cref{thm:1} reveals that the feature compression technique can also affect the resilience of distributed \dnn. Bottlenecks enhance the adversarial robustness by reducing the variance $\mathcal{O}({|\mathcal{T}||\mathcal{Y}|}/{\sqrt{n}})$ but also introduce vulnerability by unintentionally adding bias to decrease $I^*(Y;T)$. Therefore, we believe \cref{thm:1} provides a novel perspective to achieve robust distributed \dnn.

\begin{lemma}[Information Distortion]\label{def:infod}
    In a Markov chain $Y\mapsto X\mapsto T$, the conditional mutual information $I(Y;X|T)$ quantifies the distortion incurred during the transmission of information from $X$ to $T$. 
\end{lemma}
\begin{proof}
Following \cite{tishby2015deep}, we use \gls{kl} divergence $D_{KL}[P(Y|X)||P(Y|T)]$ to characterize the distortion in \gls{ib} framework. The expectation of $D_{KL}$ is
\begin{equation}
    \mathbb{E}\{D_{KL}\} = \sum_{X,T}P(X,T)\sum_{Y}P(Y|X)\log \frac{P(Y|X)}{P(Y|T)}. \label{eqn:ekld}
\end{equation}

By applying the conditional independence of Markov chain, \cref{eqn:ekld} can be refined to
\begin{equation}
    \begin{aligned}
        \mathbb{E}\{D_{KL}\} & = \sum P(X,T)P(Y|X,T)\log \frac{P(Y|X)}{P(Y|T)} \\
        & = \sum P(X,T,Y) \log \frac{P(Y|X)P(X|T)}{P(Y|T)P(X|T)} \\ 
        & = \sum P(X,T,Y) \log \frac{P(X,Y|T)}{P(Y|T)P(X|T)} \\
        & = I(X;Y|T). 
    \end{aligned}
\end{equation}    
\end{proof}
\begin{theorem}[Input vs Latent Robustness]\label{thm:2}
    With same level of distortion introduced by input $X_{adv}$ and latent $T_{adv}$, it has
    \begin{equation}
        I(Y;T')\leq I(Y;T_{adv}), \label{eqn:yt}
    \end{equation}
    where $T'$ denotes the latent corresponding to $X_{adv}$. 
\end{theorem}

\begin{proof}
    Due to the chain rule of mutual information, 
    \begin{equation}
        I(X;Y|T) = I(X,T;Y) - I(Y;T). \label{eqn:xyt}
    \end{equation}
    
    For a Markov chain $Y\mapsto X\mapsto T$, the joint distribution $P(X,Y,T)$ has following property
    \begin{equation}
        \begin{aligned}
            P(X,Y,T) &= P(T|X,Y)P(Y|X)P(X) \\
            &= P(T|X)P(Y|X)P(X).
        \end{aligned}
    \end{equation}
    
    Therefore, $I(X,T;Y)$ can be simplified as
    \begin{equation}
        \begin{aligned}
            I(X,T;Y) & = \mathbb{E}\left\{\log \frac{P(X,T,Y)}{P(X,T)P(Y)}\right\} \\
            & = \mathbb{E}\left\{\log \frac{P(T|X)P(Y|X)P(X)}{P(T|X)P(X)P(Y)}\right\} \\
            & = \mathbb{E}\left\{\log \frac{P(Y|X)}{P(Y)}\right\} = I(X;Y). \label{eqn:xy}
        \end{aligned}
    \end{equation}
    
    From \cref{eqn:xyt} and \cref{eqn:xy}, it follows that
    \begin{equation}
        I(X;Y|T) = I(X;Y) - I(Y;T). \label{eqn:residualinfo}
    \end{equation}
    
    Let $Y\mapsto X_{adv}\mapsto T'$ be the Markov chain of adversarial attacks in input space, the distortion introduced by $X_{adv}$ is 
    \begin{equation}
        I(X_{adv};Y|T') = I(X_{adv};Y) - I(Y;T'),
    \end{equation}
    where $X_{adv}$ represents adversarial samples and $T'$ is the corresponding latent representation. 
    
    Similarly, for a Markov chain $Y\mapsto X\mapsto T_{adv}$, we have
    \begin{equation}
        I(X;Y|T_{adv}) = I(X;Y) - I(Y;T_{adv}),
    \end{equation}
    where $T_{adv}$ and $I(X;Y|T_{adv})$ represents the latent perturbation and corresponding information distortion. 
    
    Assume $X_{adv}$ and $T_{adv}$ have same level of distortion, 
    \begin{equation}
        I(X_{adv};Y) - I(Y;T') = I(X;Y) - I(Y;T_{adv}). 
    \end{equation}
    
    Since $X_{adv}$ is a mapping of $X$, there is a Markov chain $Y\mapsto X\mapsto X_{adv}$.
    By \gls{dpi}, it follows that 
    \begin{equation}
        I(X;Y)\geq I(X_{adv};Y).
    \end{equation}
    
    Therefore, \cref{eqn:yt} holds.
\end{proof}

\cref{thm:2} formally describes that with same level of information distortion, attacking the latent space is less effective than attacking the input space. An important corollary can be derived from \cref{thm:2} by generalizing the input space to any latent space that is before $T$ in the Markov chain.

\begin{corollary}[Depth-Robustness Trade-Off]\label{thm:2-1}
    Assume same level of distortion is introduced to two latent space $T_{i-1}$ and $T_i$ where $i$ denotes the depth of \dnn layers, it has
    \begin{equation}
        I(Y;T_i') \leq I(Y;T^i_{adv}),
    \end{equation}
    where $T_i'$ denotes the latent at layer $i$ after introducing perturbation to $T_{i-1}$ while $T^i_{adv}$ denotes the latent at layer $i$ directly perturbed by distortion.
\end{corollary}

The proof is the same as \cref{thm:2} by replacing the Markov chain $Y\mapsto X\mapsto T$ to $Y\mapsto T_{i-1}\mapsto T_i$. \cref{thm:2-1} describes that the latent space in early layers will be more vulnerable than the the deeper latent space. This is because a small information loss gets amplified during propagation. We remark that this theoretical finding is in line with \cite{hong2020panda}, which states that the difference between adversarial and benign samples is less significant in early layers compared to deeper ones, thus enabling the early exits.

\cref{thm:2-1} is critical to distributed \dnns as finding the optimal splitting point (depth of latent space) of \dnn is one of the fundamental research questions to optimize the computation and communication overhead of distributed computing \cite{matsubara2022bottlefit,matsubara2022sc2}. However, we theoretically demonstrate that the splitting point is also a key factor of robustness which needs to be considered in design.

\subsection{Connect Theory to Experiments} \label{mi_est}


One of the open challenges in \gls{ib} is to properly estimate the mutual information between high dimensional data. Existing literature usually involves simple \dnn architecture and small datasets~\cite{tsipras2018robustness,saxe2019information}. To this end, our large-scale empirical study in \cref{sec:experiments} is based on end-to-end performance rather than \gls{ib}. In order to connect the theoretical framework to experiments, we perform an experiment on MNIST with a simple \dnn that consists of 3 layers of \glspl{cnn} followed by a linear layer. Each \gls{cnn} layer is followed by a batch normalization and a ReLU activation. Maxpooling layers are applied after the first and second \gls{cnn} layer to downsample the features and a global average pooling layer is applied before the linear layer. The simple \dnn achieved near perfect accuracy (99.42\%) on MNIST. 

\begin{table}[t]
    \centering
    \caption{Estimated mutual information $\hat{I}(Y;T)$ and classification accuracy caused by input and latent perturbations as a Function of $\epsilon$ on MNIST.}
    \resizebox{\columnwidth}{!}
    {\begin{tabular}{ccccccccccccc}
        \toprule
         & \multicolumn{2}{c}{MINE \cite{belghazi2018mutual}} & \multicolumn{2}{c}{NWJ \cite{nguyen2010estimating}} & \multicolumn{2}{c}{CPC \cite{oord2018representation}} & \multicolumn{2}{c}{CLUB \cite{cheng2020club}} & \multicolumn{2}{c}{DoE \cite{mcallester2020formal}} & \multicolumn{2}{c}{Accuracy} \\
         & Input & Latent & Input & Latent & Input & Latent & Input & Latent & Input & Latent & Input & Latent \\
        \midrule
        $\epsilon=0.01$ & 2.01 & \textbf{2.08} & 1.28 & \textbf{1.31} & 2.25 & \textbf{2.26} & 9.88 & \textbf{9.88} & 27.44 & \textbf{27.88} & 98.68\% & \textbf{98.89\%} \\
        $\epsilon=0.02$ & 2.01 & \textbf{2.06} & 1.25 & \textbf{1.28} & 2.23 & \textbf{2.24} & 9.82 & \textbf{9.83} & 26.47 & \textbf{27.37} & 97.09\% & \textbf{97.89\%} \\
        $\epsilon=0.03$ & 1.93 & \textbf{1.95} & 1.17 & \textbf{1.26} & 2.19 & \textbf{2.22} & 9.74 & \textbf{9.77} & 25.28 & \textbf{26.88} & 94.27\% & \textbf{96.24\%} \\
        $\epsilon=0.04$ & 1.83 & \textbf{1.98} & 1.10 & \textbf{1.22} & 2.16 & \textbf{2.20} & 9.64 & \textbf{9.69} & 23.82 & \textbf{26.26} & 88.79\% & \textbf{93.40\%} \\
        $\epsilon=0.05$ & 1.78 & \textbf{1.88} & 1.01 & \textbf{1.16} & 2.11 & \textbf{2.16} & 9.51 & \textbf{9.59} & 22.03 & \textbf{25.61} & 78.64\% & \textbf{88.63\%} \\
        $\epsilon=0.06$ & 1.71 & \textbf{1.85} & 0.89 & \textbf{1.12} & 2.04 & \textbf{2.12} & 9.36 & \textbf{9.46} & 19.88 & \textbf{24.86} & 56.39\% & \textbf{79.96\%} \\
        $\epsilon=0.07$ & 1.62 & \textbf{1.80} & 0.79 & \textbf{1.12} & 1.94 & \textbf{2.07} & 9.19 & \textbf{9.32} & 17.13 & \textbf{24.06} & 31.59\% & \textbf{67.98\%} \\
        $\epsilon=0.08$ & 1.48 & \textbf{1.79} & 0.67 & \textbf{1.03} & 1.83 & \textbf{2.00} & 9.01 & \textbf{9.16} & 13.77 & \textbf{23.12} & 13.54\% & \textbf{52.26\%} \\
        $\epsilon=0.09$ & 1.35 & \textbf{1.65} & 0.56 & \textbf{1.00} & 1.69 & \textbf{1.91} & 8.81 & \textbf{8.98} & 9.83 & \textbf{22.22} & 4.34\% & \textbf{34.75\%} \\
        $\epsilon=0.10$ & 1.19 & \textbf{1.62} & 0.45 & \textbf{0.97} & 1.51 & \textbf{1.81} & 8.61 & \textbf{8.77} & 5.08 & \textbf{21.17} & 1.22\% & \textbf{20.58\%} \\
        \bottomrule
    \end{tabular}}\label{tab:mine}
\end{table}

We consider the feature map output by the 2nd \gls{cnn} layer as the target latent space and $l_\infty$ constrained \pgd \cite{madry2017towards} with a perturbation strength $\epsilon$ is applied to both input and latent space. To have a comprehensive evaluation, multiple commonly used mutual information estimators are applied \cite{belghazi2018mutual,nguyen2010estimating,oord2018representation,mcallester2020formal,cheng2020club}. Among them, MINE \cite{belghazi2018mutual}, NWJ \cite{nguyen2010estimating}, and CPC \cite{oord2018representation} are \textit{lower bounded} estimators while DoE \cite{mcallester2020formal} and CLUB \cite{cheng2020club} are \textit{upper bounded} estimators. 

\cref{tab:mine} demonstrates the results of estimated mutual information $\hat{I}(Y;T)$ and corresponding classification accuracy as a function of noise level $\epsilon$. Note that the estimated mutual information varies significantly across different estimators, as they are built on different lower or upper bounds. While these estimations cannot give us exact value of mutual information, all of them show a same trend. To this end, we consider these estimations to be valid. As shown in \cref{tab:mine}, with an increasing perturbation budget $\epsilon$, all estimated $\hat{I}(Y;T)$ as well as the accuracy of CNN decrease. This indicates that the $I(Y;T)$ is an appropriate tool to quantify the \dnn robustness as proposed in Lemma~\ref{def:lr}. Moreover, the mutual information as well as the classification accuracy associated with latent perturbations shows a higher value compared to those of input perturbations, which supports \cref{thm:2}. This experiment bridges the theoretical analysis to large scale experiments in the rest of paper.

\section{Experimental Setup}\label{sec:exp-setup}

\subsection{Attacks Under Consideration}
Adversarial attacks can be categorized as gradient-based, score-based and decision-based approaches. Gradient-based attacks craft the adversarial samples by maximizing classification loss with gradient updates. On the other hand, score-based attacks leverage the output score of \dnn as feedback to iteratively search the perturbation in the black-box setting and decision-based attacks can only access the hard label which has the highest score in the output. To have a thorough assessment, we implemented 10 popular attacks to \dnns. These include 4 gradient-based white-box attacks \cite{goodfellow2014explaining,kurakin2018adversarial,dong2018boosting,madry2017towards}, as well as 3 score-based black-box attacks \cite{ilyas2018black,li2019nattack,andriushchenko2020square} and 3 decision-based black-box attacks \cite{dong2019efficient,cheng2019sign,wang2022triangle}. 

\noindent\textbf{White-box Attacks:} We consider 4 gradient-based attacks only in the white-box setting because latent representations are different for each model, resulting in numerous surrogate \dnns training in black-box scenarios that may be infeasible in practice. We choose \fgsm \cite{goodfellow2014explaining}, \bim \cite{kurakin2018adversarial}, \mim \cite{dong2018boosting} with $l_\infty$ norm constraints. \pgd \cite{madry2017towards} is implemented for both $l_2$ and $l_\infty$ spaces as a baseline for other black-box attacks.

\noindent\textbf{Black-box Attacks:} We consider 3 score-based attacks \nes \cite{ilyas2018black}, \natk \cite{li2019nattack} and \satk \cite{andriushchenko2020square} in $l_{\infty}$ space and 3 decision-based attacks \eatk \cite{dong2019efficient}, \sopt \cite{cheng2019sign} and \tatk \cite{wang2022triangle} in $l_2$ space.

\noindent\textbf{Dataset and Metrics:} We evaluate adversarial robustness using 1000 samples from the validation set of ImageNet-1K \cite{deng2009imagenet}, limiting the samples to those which are correctly classified. We define the perturbation budget $\epsilon$ as the mean square error (MSE) under the $l_2$ norm constraint (i.e., $\epsilon\times d=\sigma^2$ and $\epsilon\times t=\sigma^2$ in input and latent space respectively where $d$, $t$ and $\sigma$ are the dimension of input, dimension of latent, and the distance constraint defined in \cref{sec:model}.) and the maximum element-wise distance under $l_\infty$ norm constraint (i.e., $\epsilon=\sigma$), we define the \gls{asr} as
\begin{equation}
    \mbox{ASR}(\epsilon)=\frac{1}{N}\sum_{i=1}^{N}\mathbf{I}\left\{\argmax_{k=1,\dots,K} f(x_i,\delta_i) \neq y_i\right\},
\end{equation}
where $\mathbf{I}\{\cdot\}$ is the indicator function and $f(x_i,\delta_i)$ is the \gls{dnn} output when fed with the $i$-th sample.

\begin{table}[t]
    \centering
    \caption{List of feature compression approaches considered in this paper}
    \resizebox{\columnwidth}{!}
    {
    \begin{tabular}{ccl}
    \toprule
    {Category} & {Approach} & \hspace{3.6cm}{Description} \\
    \midrule
    \multirow{3}{1.5cm}{Dimension Reduction} & \acrshort{sb} & naive supervised compression trained with cross entropy \cite{eshratifar2019bottlenet,shao2020bottlenet++} \\
    & \acrshort{db} & bottleneck trained with naive knowledge distillation \cite{matsubara2019distilled} \\
    & \acrshort{bf} & multi-stage training with distillation and cross entropy \cite{matsubara2022bottlefit} \\
    \midrule
    \multirow{2}{1.5cm}{Datasize Reduction} & \acrshort{jc} & reduce precision in frequency domain using JPEG approach \cite{alvar2021pareto} \\
     & \acrshort{qt} & uniformly compress every element using naive bit quantization \cite{singh2020end} \\
    \midrule
    \multirow{2}{1.5cm}{Advanced} & \multirow{2}{.4cm}{ES} & bottleneck trained with distillation and information-based loss \\
     & & and data compressed with quantization and entropy coding \cite{matsubara2022supervised} \\
    \bottomrule
    \end{tabular}}
    \label{tab:sc}
\end{table} 

\subsection{DNNs Under Consideration} \label{sec:dnn_setup}
\noindent\textbf{\dnn Architectures.} First, we consider 3 \glspl{dnn}:~VGG16 from \cite{simonyan2014very} as well as ResNet50 and ResNet152 from \cite{he2016deep}. Both VGG and ResNet have 5 feature extraction blocks which consist of several identical convolutional layers to extract features in different depths. During experiments, we consider the output of the third block as the target latent space without further specification. (We investigate the latent robustness as a function of depth in \cref{exp:depth}.) 

Next, to investigate the effect introduced by the feature compression layer (i.e., the ``\textit{bottleneck}'') proposed for distributed \dnns, we introduce the same bottleneck design as \cite{matsubara2022sc2} to VGG16, ResNet50 and ResNet152 and denote the new architectures as VGG16-fc, ResNet50-fc and ResNet152-fc.

\noindent\textbf{Compression Approaches.}~In distributed \dnns, compression can be categorized as dimension reduction and data size reduction. Dimension reduction applies the bottleneck layer to compress the cardinality of the latent space while data size reduction leverages lossless (e.g. entropy coding) or lossy (e.g. quantization, JPEG) coding to minimize the bit rate of latent representations. Note that dimension reduction is a standalone approach that can applied on top of the data size reduction and vise versa. 

We consider 3 different bottleneck training strategies for dimension reduction: \gls{sb} \cite{eshratifar2019bottlenet,shao2020bottlenet++}, \gls{db} \cite{matsubara2019distilled} and \gls{bf} \cite{matsubara2022bottlefit}. We also choose 2 data size reduction strategies \gls{jc} \cite{alvar2021pareto}, \gls{qt}~\cite{singh2020end} as well as 1 advanced approach \gls{es} that both compress the dimension and data size \cite{matsubara2022supervised}. We summarize these approaches in \cref{tab:sc}.

\section{Experimental Results}\label{sec:experiments}

\subsection{Performance w.r.t DNN Architecture}\label{sec:performance_architecture}

We first assess the latent robustness against different attack algorithms. \cref{fig:bf152} shows the \gls{asr} obtained on ResNet152-fc with perturbation budget $\epsilon = 0.01$. Remarkably, we notice that the \gls{asr} is higher for attacks in the input space than attacks in the latent space for each attack algorithm considered. In the case of \tatk, the latent \gls{asr} is 88\% less than the input \gls{asr}. On average, the \gls{asr} in input is 57.49\% higher than the \gls{asr} obtained by attacks in the latent space. Moreover,  \satk, \eatk, and \tatk have lowest \gls{asr} on latent representations. This is because these attacks search perturbations in a lower dimensional space, and hence it is more challenging for the adversary to find the effective distortions in compressed latent space. \cref{fig:bf152} support the finding of \cref{thm:2}, which states that under the same distortion constraint, the latent space constantly demonstrates better robustness than the input space.

\begin{figure}[t]
    \centering
    \includegraphics[width=\columnwidth]{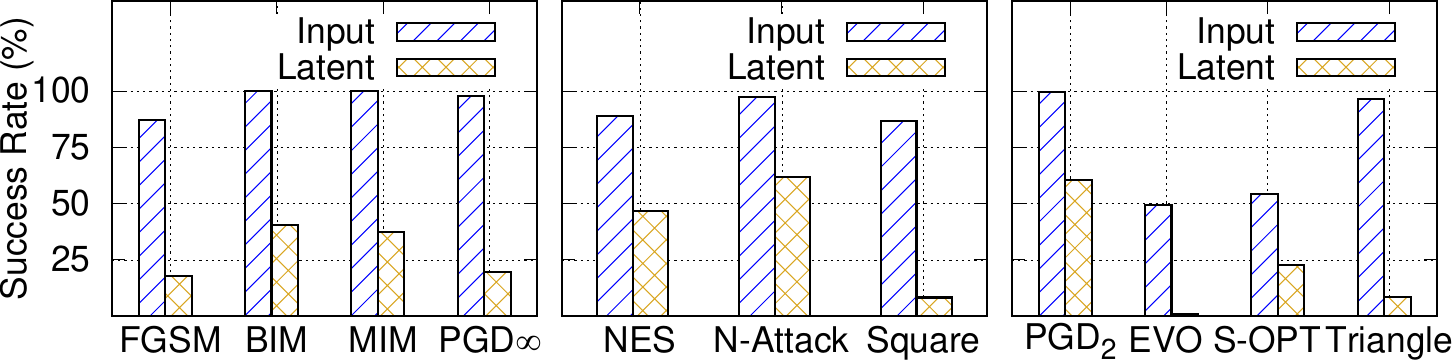}
    \caption{10 different attacks to ResNet152-fc with perturbation budget $\epsilon = 0.01$.}
    \label{fig:bf152}
\end{figure}

As depicted in \cref{fig:bf152}, although various attacks achieved different \gls{asr}, they also demonstrated a similar trend in performance loss. For this reason, in the following experiments we only show results of selected attacks for briefness. To investigate the effect of different \dnn architectures, we choose \pgd as the white-box baseline, \natk as the score-based attack, and \tatk as the representative of decision-based attack. \cref{fig:md} shows the \gls{asr} of \pgd, \natk and \tatk on different \dnns with $\epsilon=0.003$. For each \dnn, the \gls{asr} is higher in input-space attacks. In VGG16-bf, which shows the best robustness, the average \gls{asr} in the latent space is 87.8\% lower than input attacks. On average, latent representations are 58.33\% more robust. Therefore, the experiments in \cref{fig:bf152,fig:md} demonstrate that \cref{thm:2} is valid across various adversarial attacks and \dnn architectures.

\subsection{Performance w.r.t Compression Approach} \label{sec:exp2}

\begin{figure}[t]
    \centering
    \includegraphics[width=\columnwidth]{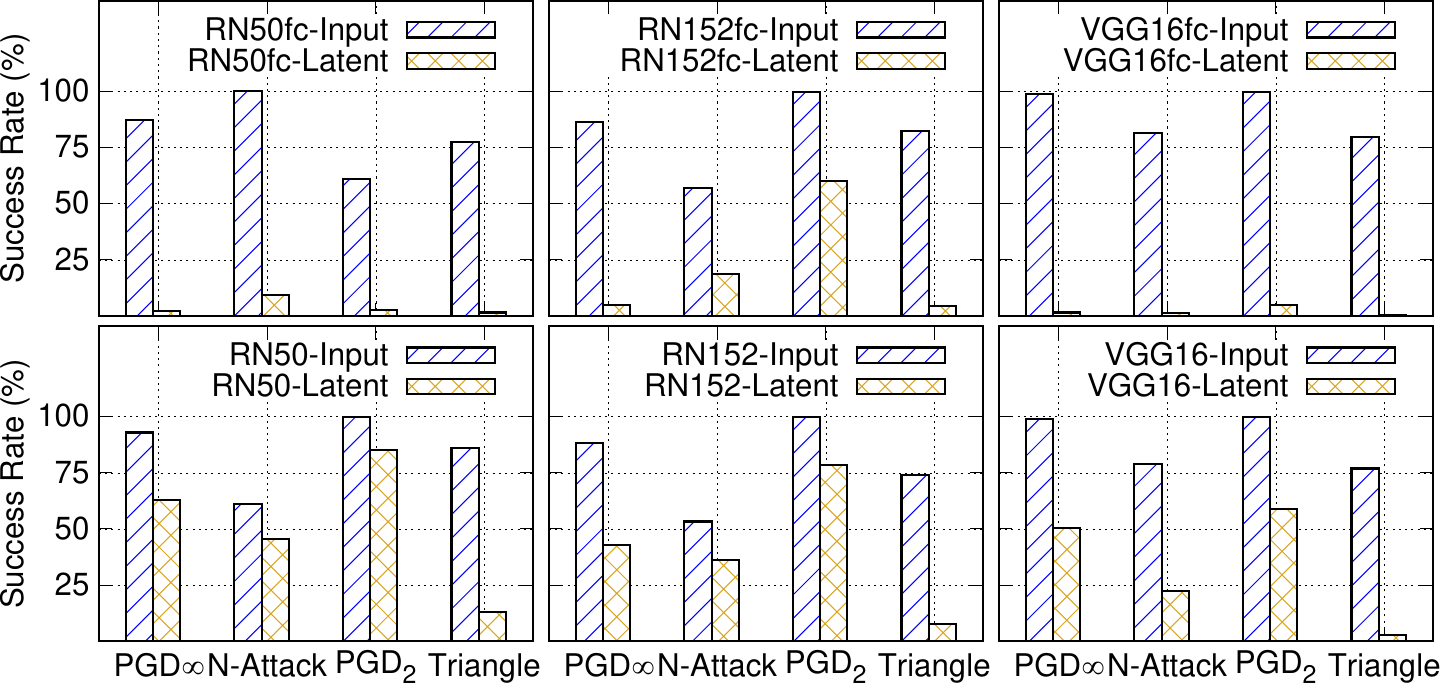}
    \caption{Whitebox baseline (PGD) and blackbox attacks under $l_\infty$ and $l_2$ in input and latent space with perturbation budget $\epsilon = 0.003$ applied to 6 different \dnns.}
    \label{fig:md}
\end{figure}

To evaluate the robustness of different compression approaches, we choose \satk and \tatk which are more recent approaches for score-based and decision-based attacks respectively. We do not consider gradient-based attacks as compression approaches such as \gls{db}, \gls{qt}, \gls{es} can lead to gradient masking. Hence, their robustness cannot be correctly evaluated by naive gradient-based attacks \cite{athalye2018obfuscated}. We choose a larger perturbation budget ($\epsilon = 0.05$) than the experiments depicted in \cref{fig:bf152} to further evaluate whether the compressed feature space is robust to attacks relying on low-dimensional sub-space searching. We note that data compression can be applied in addition to bottlenecks. However, for ablation study purposes, we choose ResNet50 without bottlenecks for \gls{jc} and \gls{qt} and ResNet50-fc for others. 

\cref{fig:sc} shows the \gls{asr} of \satk in $l_\infty$ space and \tatk in $l_2$ space with perturbation budget $\epsilon = 0.05$. Except the \gls{jc} and \gls{qt}, the adversarial robustness shows the same trend regardless of the examined approaches. The average \glspl{asr} in input space are 79.07\% and 87.22\% higher than the average \glspl{asr} in latent space for \satk and \tatk respectively. For \glspl{dnn} with bottlenecks, despite the increase in perturbation, the \gls{asr} of \satk and \tatk performed in latent representations do not increase distinctively comparing to \cref{fig:bf152}. However, since  \gls{jc} and \gls{qt} do not have separate feature compression layers, the \gls{asr} of \satk and \tatk in input space are only 55.8\%, 0.25\% higher than the attacks in latent space, showing a significant downgrade comparing to the other 4 approaches. These results confirm that the compressed feature space is indeed robust to attacks that search in lower dimensions.

\begin{figure}[t]
    \centering
    \includegraphics[width=\columnwidth]{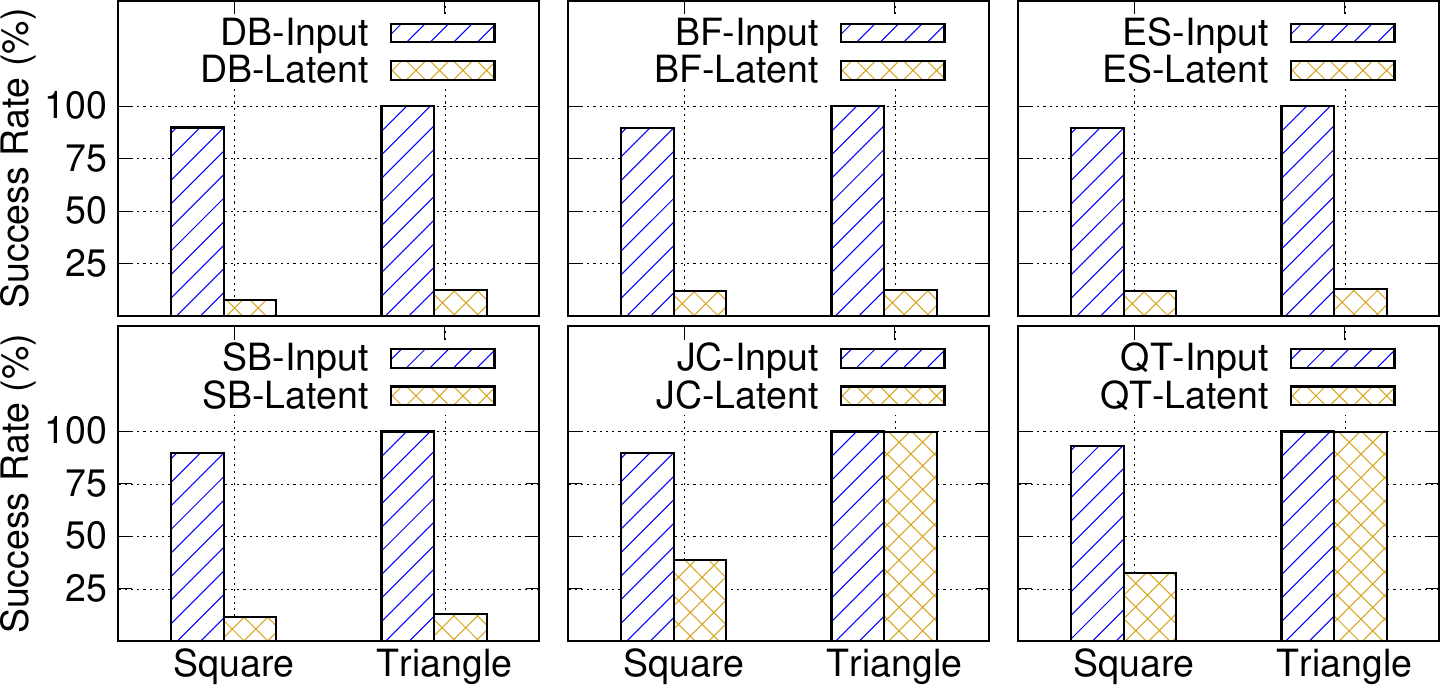}
    \caption{Square and Triangle attack success rate associated with 6 feature compression approaches with perturbation budget $\epsilon = 0.05$.}
    \label{fig:sc}
\end{figure}

\subsection{Performance w.r.t Dimension} \label{exp_dim}

In the previous experiment, we have shown that the robustness of the latent representation is mostly characterized by the bottleneck layer properties rather than the data size compression approach. Thus, we further evaluate the robustness for different sizes of latent space using \natk, \mim under $l_\infty$ constraint and \sopt, \pgd under $l_2$ constraint with multiple perturbation budgets ($\epsilon = 0.003$; $\epsilon = 0.01$; $\epsilon = 0.03$). The cardinality of the latent space is controlled by the number of channels at the bottleneck layer. We first set the channel number as 12 for ResNet152-fc that can achieve 77.47\% validation accuracy, which is almost similar to the performance of the original ResNet152 (78.31\%). Then, we reduce the channel size to 3, which decreases the dimension of latent representations but also reduces the end-to-end performance to 70.01\%. We do not repeat the results for \satk and \tatk since they fail to achieve satisfactory \gls{asr} in the previous experiments due to their smaller search subspace, as shown in \cref{fig:bf152,fig:sc}. 

\cref{fig:ch} shows results obtained by considering the $l_\infty$ and $l_2$ attacks with multiple perturbation budgets ($\epsilon = 0.003$; $\epsilon = 0.01$; $\epsilon = 0.03$) in the latent space of original ResNet-152, 12-channel ResNet152-fc and 3-channel ResNet152-fc. From ResNet152 to 12-channel ResNet152-fc, the \gls{asr} reduces as the dimensionality of latent representations decreases. However, after reducing the channel size to 3, the \gls{asr} does not decrease any further. Conversely, distributed \dnns with a smaller channel size become more vulnerable to perturbations. (One exception is \mim with $\epsilon=0.03$, where the \gls{asr} of the 12-channel model is $79.5\%$ and the \gls{asr} of the 3-channel model is $75.0\%$. However, due to the difference of devices and random seeds, the \gls{asr} can vary 2-3\%. Thus we do not consider the decrease of the \mim success rate in 3-channel ResNet152-fc which is less than 5\%.) This is because when reducing channels from 12 to 3 channels, the accuracy also decreases to 7.46\%, which in turn lessens the end-to-end generalization capability (i.e., $I^*(Y;T)$). This experiment supports our analysis in \cref{thm:1} that the robustness in latent representations of distributed \gls{dnn} is jointly determined by the end-to-end performance and feature dimensions.

\begin{figure}[t]
    \centering
    \includegraphics[width=\columnwidth]{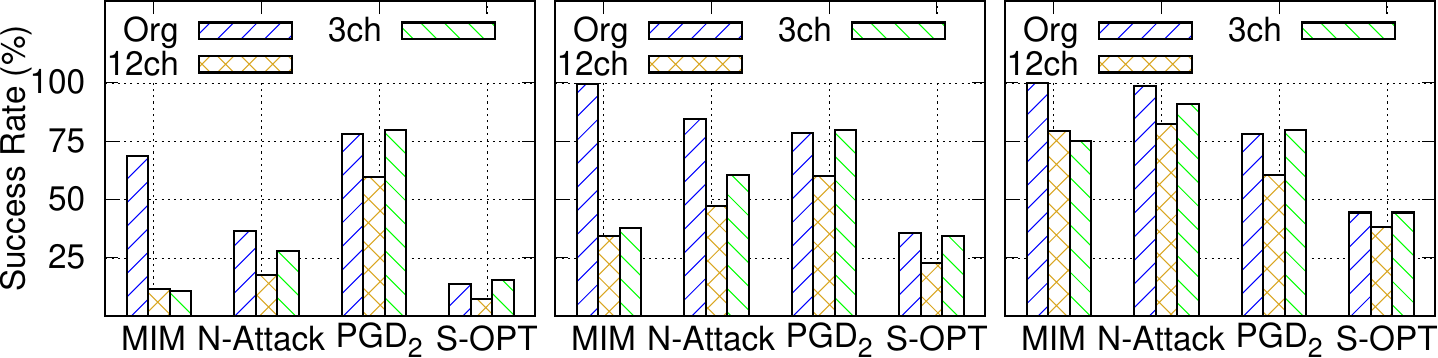}
    \caption{$l_\infty$ and $l_2$ attack success rate for different latent cardinalities of ResNet152-fc with different perturbation budgets: (left) $\epsilon = 0.003$; (center) $\epsilon = 0.01$; (right) $\epsilon = 0.03$.}
    \label{fig:ch}
\end{figure}

\subsection{Performance w.r.t Depth} \label{exp:depth}

To validate \cref{thm:2-1}, we evaluate the robustness of latent representations at different depths using two architectures: VGG16 as a shallow model and ResNet152 as a deeper model. Both architectures contain five feature extraction blocks, each comprising multiple convolutional layers that extract features at different depths. We select three splitting points for our analysis: the outputs of the first, third, and fifth blocks (denoted as Feature 0, Feature 2, and Feature 4, respectively). For a comprehensive evaluation, we employ \pgd as a white-box attack and \satk as a black-box attack, testing multiple perturbation budgets $\epsilon$.

\cref{tab:depthvgg} demonstrates that latent representations consistently exhibit greater robustness than input-level perturbations, regardless of splitting depth or perturbation strength\textemdash corroborating our findings in \cref{sec:performance_architecture}. Notably, the \gls{asr} decreases monotonically with feature depth for both VGG16 and ResNet152, indicating that deeper splitting points yield more robust distributed \dnns.

\rev{To quantify the computational trade-off associated with this robustness gain, we deployed the distributed \dnns on a realistic edge computing setup: a Jetson Orin Nano serves as the mobile device executing the head model, while a NVIDIA A100 acts as the edge server executing the tail model. \cref{tab:latency_profiling} presents the latency profiling results for batch processing of 32 inputs. Each row corresponds to a specific architecture and splitting configuration, with the ``Head" and ``Tail" columns reporting execution times on the mobile device and edge server, respectively. } 

\rev{The results reveal a clear trend: deeper splitting points increase the head execution latency, reflecting the increased computational burden on the mobile device. This empirical evidence demonstrates a fundamental trade-off in distributed \dnn design. Splitting at later layers enhances robustness against latent-space adversarial perturbations but incurs higher computational costs on resource-constrained mobile devices. Conversely, early-layer splitting minimizes the mobile computational load but leaves the system more vulnerable to adversarial attacks. These findings empirically validate the theoretical framework presented in \cref{thm:2-1}, highlighting the need for application-specific optimization when designing distributed inference systems.}

\begin{table}[t]
    \caption{\gls{asr} as a function of depth}
    \begin{center}
    \resizebox{\columnwidth}{!}
    {
    \begin{tabular}{cccccccccc}
    \toprule
     & & \multicolumn{2}{c}{Input} & \multicolumn{2}{c}{Feature 0} & \multicolumn{2}{c}{Feature 2} & \multicolumn{2}{c}{Feature 4}  \\
     & & \pgd & Square & \pgd & Square & \pgd & Square & \pgd & Square \\
    \midrule
    & $\epsilon=0.003$ & 98.8\% & 48.2\% & 50.3\% & 3.4\% & 22.3\% & 3.4\% & 6.5\% & 0.6\% \\
    & $\epsilon=0.006$ & 99.3\% & 79.1\% & 81.6\% & 9.3\% & 50.1\% & 8.6\% & 12.4\% & 1.4\% \\
    VGG16 & $\epsilon=0.009$ & 99.8\% & 91.2\% & 93.7\% & 13.4\% & 71.5\% & 12.5\% & 19.4\% & 2.5\% \\
    & $\epsilon=0.012$ & 99.7\% & 95.4\% & 96.9\% & 17.3\% & 87.2\% & 16.0\% & 27.5\% & 3.7\% \\
    & $\epsilon=0.015$ & 99.9\% & 98.8\% & 98.5\% & 22.1\% & 93.5\% & 20.0\% & 35.2\% & 4.5\% \\
    \midrule
    & $\epsilon=0.003$ & 88.1\% & 36.4\% & 65.5\% & 14.7\% & 43.1\% & 3.7\% & 1.6\% & 0.0\% \\
    & $\epsilon=0.006$ & 96.9\% & 63.3\% & 90.5\% & 31.0\% & 76.7\% & 8.4\% & 3.4\% & 0.2\% \\
    ResNet152& $\epsilon=0.009$ & 98.4\% & 81.9\% & 96.9\% & 45.7\% & 92.9\% & 13.4\% & 4.9\% & 0.3\% \\
    & $\epsilon=0.012$ & 99.0\% & 90.0\% & 99.0\% & 56.7\% & 97.4\% & 17.7\% & 6.0\% & 0.6\% \\
    & $\epsilon=0.015$ & 99.4\% & 94.7\% & 99.1\% & 68.8\% & 98.5\% & 21.9\% & 7.3\% & 0.7\% \\
    \bottomrule
    \end{tabular}
    }
    \end{center}
    \label{tab:depthvgg}
\end{table}

\begin{table}[t]
    \caption{Latency profiling of DNNs as a function of depth}
    \begin{center}
    \resizebox{\columnwidth}{!}
    {
    \begin{tabular}{cccccc}
    \toprule
     & & Head (ms) & Tail (ms) & Comm. (ms) & Overall (ms) \\
    \midrule
    & Feature 0 & 174.03 $\pm$ 17.11 & 13.42 $\pm$ 0.14 & 2211.61 $\pm$ 21.87 & 2399.06 $\pm$ 27.77 \\
    VGG16 & Feature 2 & 386.77 $\pm$ 16.16 & 5.32 $\pm$ 0.06 & 63.50 $\pm$ 1.80 & 455.59 $\pm$ 16.26 \\
    & Feature 4 & 525.59 $\pm$ 17.90 & 0.84 $\pm$ 0.01 & 1.92 $\pm$ 0.09 & 528.35 $\pm$ 17.90 \\
    \midrule
    & Feature 0 & 24.22 $\pm$ 2.88 & 23.33 $\pm$ 0.09 & 67.66 $\pm$ 2.14 & 115.21 $\pm$ 3.59 \\
    ResNet152 & Feature 2 & 217.57 $\pm$ 4.80 & 15.80 $\pm$ 0.07 & 458.26 $\pm$ 5.53 & 691.63 $\pm$ 7.32 \\
    & Feature 4 & 611.26 $\pm$ 4.85 & 0.03 $\pm$ 0.00 & 0.12 $\pm$ 0.03 & 611.41 $\pm$ 4.85 \\
    \bottomrule
    \end{tabular}
    }
    \end{center}
    \label{tab:latency_profiling}
\end{table}

\subsection{\rev{Evaluation on Other Tasks}} \label{sec:dataset}

\rev{In this experiment, we provide a comprehensive assessment of our theoretical findings across different tasks. We consider two datasets: RAVDESS \cite{livingstone2018ryerson} which is commonly used in speech emotion recognition, and RadioML \cite{o2018over} for automatic modulation classification. For both tasks we use the same neural network architecture which is summarized in Table~\ref{tab:model}. For each conv layer, it follows by a batch normalization and ReLU activation.}

\begin{wraptable}{l}{0.3\linewidth}
    \centering
    \caption{Summary of the \gls{dnn} Classifier}
    \begin{tabular}{|c|}
    \hline
        Conv 1$\times$3, 64 \\
        Conv 1$\times$3, 64 \\
        \hline
        Maxpool 1$\times$2 \\
        \hline
        Conv 1$\times$3, 128 \\
        Conv 1$\times$3, 128 \\
        \hline
        Maxpool 1$\times$2 \\
        \hline
        Conv 1$\times$3, 256 \\
        Conv 1$\times$3, 256 \\
        \hline
        Avgpool \\
        \hline
        Linear 256$\times$n \\
    \hline
    \end{tabular}
    \label{tab:model}
\end{wraptable}

\rev{For RAVDESS, we use the complete dataset and split it into training and testing sets with an 8:2 ratio. We employ an SGD optimizer with a learning rate of 0.01, weight decay of 5e-4, and momentum of 0.9. The model is trained for 200 epochs with a batch size of 32. Additionally, since RAVDESS is a small dataset, we apply label smoothing to prevent overfitting.}

\rev{For RadioML, we use only data with SNR ranging from 10 dB to 20 dB, as signals at low SNR are challenging for classifiers to identify. We split the dataset into training and testing sets using the same configuration as RAVDESS. We employ the Adam optimizer with a learning rate of 0.001 and train for 50 epochs with a batch size of 1024. We do not apply label smoothing since RadioML is a large dataset.}

\rev{After training, we split the model at the 4th conv layer and apply \fgsm with $\epsilon=0.01$ to both input and latent space. Table~\ref{tab:dataset} summarizes the classification accuracy on clean and adversarial data. The accuracy on latent adversarial samples consistently outperforms the performance on input adversarial samples across all datasets. This supports our theoretical analysis in Section~\ref{sec:theory}.}

\begin{table}[h]
    \centering
    \caption{Input vs Latent Attack for Different Datasets}
    \begin{tabular}{cccc}
    \toprule
    Dataset & Clean & Atk (input) & Atk (latent) \\
    \midrule
    RAVDESS \cite{livingstone2018ryerson} & 82.25\% & 45.73\% & 65.87\% \\
    RadioML \cite{o2018over} & 92.51\% & 58.34\% & 60.68\% \\
    \bottomrule
    \end{tabular}
    \label{tab:dataset}
\end{table}

\section{Related Work}\label{sec:rw}

\noindent\rev{\textbf{Adversarial Attacks.} Adversarial attacks can be categorized as gradient-based, score-based, and decision-based. In \textit{gradient-based scenarios}, attackers can obtain the input gradient through backpropagation and craft adversarial samples with gradient steps. \fgsm~\cite{goodfellow2014explaining} crafts adversarial samples in the $\mathit{l_{~\infty}}$ space based on the one-step input gradient sign. \bim ~\cite{kurakin2018adversarial} increases the effectiveness of \fgsm by iteratively updating adversarial samples over multiple gradient steps. \mim ~\cite{dong2018boosting} introduces momentum to iterative attacks which improves the transferability of adversarial samples. \pgd ~\cite{madry2017towards} generalizes iterative attacks to $\mathit{l_{~p}}$ space with a random start. \textit{Score-based adversaries} can only access the scores for every class given by the \dnn. \nes~\cite{ilyas2018black} applies evolutionary algorithm to estimate gradient within limit queries. \natk~\cite{li2019nattack} designs a learnable Gaussian distribution to generate effective perturbations. \satk~\cite{andriushchenko2020square} adds a localized square-shaped perturbation at a random position to the input in each iteration. \textit{Decision-based attacks} assume the adversary is only aware of the label having the highest score in the \dnn output. \eatk~\cite{dong2019efficient} minimizes perturbations with heuristic search. \sopt~\cite{cheng2019sign} accelerates the convergence with estimated gradient sign, while \tatk~\cite{wang2022triangle} minimizes perturbations in a low frequency space with the geometric property in the decision boundary.}

\noindent\textbf{Empirical and Theoretical Study of \dnn Robustness.} \cite{carlini2019evaluating} propose a set of criteria to study adversarial robustness with empirical results, and \cite{dong2020benchmarking,croce2020robustbench} evaluate the robustness of different defense approaches. However, in the absence of a theoretical framework, different studies may arrive at conflicting conclusions. For example, \cite{su2018robustness} argue that there is a trade-off between generalization and robustness while \cite{stutz2019disentangling} state that generalization does not affect the robustness. On the other hand, \gls{pac} learning has been used to theoretically analyze the adversarial robustness of \dnns ~\cite{montasser2019vc,bhattacharjee2023robust}. However, these approaches are evaluated on simple \dnns, while we evaluate our findings on state-of-the-art \dnns.

\noindent\textbf{Information Bottleneck.} \cite{tishby2015deep} first propose to use \Gls{ib} \cite{tishby2000information} to analyze \dnns and \cite{shwartz2017opening,saxe2019information} analyze the generalization and compression capability of \dnns with experiments. However, due to the challenge of mutual information estimation \cite{poole2019variational}, such studies only perform experiments on relatively small neural networks. Recent work applying \gls{ib} for robust learning \cite{alemi2016deep,kim2021distilling} 
has shown significant enhancements in experimental results. However, it does not involve a rigorous theoretical analysis for robustness of \dnns. In contrast, we provide a theoretical framework to analyze the robustness of distributed DNNs in split computing as well as provide a comprehensive experimental assessment in large scale.

\section{\rev{Limitation and Future Endeavor}} \label{sec:discuss}

\noindent\rev{\textbf{Mutual Information Estimation.} We employ \gls{ib} as our theoretical framework. However, there has been extensive debate regarding whether \gls{ib} constitutes an appropriate tool for explaining and interpreting the properties of \dnns \cite{shwartz2017opening,saxe2019information}, as accurately quantifying mutual information in high-dimensional random variables presents significant challenges. In this work, we present a small-scale mutual information estimation experiment in Section~\ref{mi_est} and evaluate the large scale experiment in Section~\ref{sec:experiments} using end-to-end performance metrics rather than mutual information. Thus, this work can be more rigorous with advanced mutual information experiments in larger scale.}

\noindent\rev{\textbf{Generalization of Theoretical Findings}. Our experiments focus exclusively on classification tasks. Investigating the framework's applicability to other non-classification tasks would be valuable for future work. Furthermore, our experiments are performed on conventionally trained neural networks while evaluation on adversarially trained models \cite{madry2017towards} or dynamically robust models \cite{zhang2024hyperadv} remains unexplored.}

\noindent\rev{\textbf{Joint Optimization of Efficiency and Robustness}. Note that the primary goal of this paper is to establish a fundamental understanding of robustness in the latent space of split computing. An important future direction is to develop an optimization framework that jointly optimizes task performance, end-to-end latency, and adversarial robustness in the latent space.}

\section{Concluding Remarks}\label{sec:conclusions}

This paper has investigated adversarial attacks to latent representations of \dnns. First, we have theoretically analyzed the robustness of latent representations based on the information bottleneck theory. To prove our theoretical findings, we have performed an extensive set of experiments with 6 different \gls{dnn} architectures, 6 different distributed \gls{dnn} approaches and considering 10 different attacks in literature. Our investigation concludes that latent representations in deeper layers are more robust than those in early layers assuming the same level of information distortion. Moreover, the adversarial robustness in latent space is jointly determined by the feature size and the end-to-end model generalization capability. We hope that this work will inspire future work on the topic of adversarial machine learning on distributed \dnns. 

\section*{Acknowledgment}

This work has been supported in part by the National Science Foundation under grants CNS-2312875 and OAC-2530896; by the Air Force Office of Scientific Research under grant FA9550-23-1-0261; by the Office of Naval Research under grant N00014-23-1-2221.

{
\bibliographystyle{elsarticle-num-names.bst}
\bibliography{bibliography,rev}
}

\end{document}